\documentclass[onecolumn, 11pt, letterpaper]{article}

\usepackage{algorithm}
\usepackage{algorithmicx}
\usepackage[noend]{algpseudocode}
\usepackage{amsmath}\allowdisplaybreaks
\usepackage{amssymb}
\usepackage{amsthm}
\usepackage{color,xcolor}
\usepackage[margin = 1in]{geometry}
\usepackage{hyperref}
\usepackage{ifthen}
\usepackage{mathtools}
\usepackage{braket}
\usepackage{authblk}

\usepackage{graphicx}
\usepackage{subfigure}
\usepackage{caption}
\usepackage{bbding}

\hypersetup{colorlinks=true, linkcolor=blue, citecolor=blue, anchorcolor=blue, urlcolor=blue}  

\newtheorem{theorem}{Theorem}
\newtheorem{lemma}{Lemma}
\newtheorem{definition}{Definition}

\newtheorem{corollary}{Corollary}
\newtheorem{example}{Example}

\newcommand{\eq}[1]{(\ref{eq:#1})}

\newcommand{\lem}[1]{\hyperref[lem:#1]{Lemma~\ref*{lem:#1}}}
\newcommand{\thm}[1]{\hyperref[thm:#1]{Theorem~\ref*{thm:#1}}}
\newcommand{\algo}[1]{\hyperref[algo:#1]{Algorithm~\ref*{algo:#1}}}
\newcommand{\cor}[1]{\hyperref[cor:#1]{Corollary~\ref*{cor:#1}}}
\newcommand{\secc}[1]{\hyperref[sec:#1]{Section~\ref*{sec:#1}}}

\def\>{\rangle}
\def\<{\langle}

\renewcommand\appendix{\par
    \setcounter{section}{0}
    \setcounter{subsection}{0}
    \gdef\thesection{\Alph{section}.}}
\title{Bounded Memory Adversarial Bandits with Composite Anonymous Delayed Feedback}
\author[1,2]{Zongqi Wan\footnote{Email:wanzongqi20s@ict.ac.cn}}
\author[1,2]{Xiaoming Sun\footnote{Email:sunxiaoming@ict.ac.cn}}
\author[1,2]{Jialin Zhang\footnote{Email:zhangjialin@ict.ac.cn}}
\affil[1]{Institute of Computing Technology, Chinese Academy of Sciences, Beijing, China}
\affil[2]{University of Chinese Academy of Sciences, Beijing, China}
\date{}

\begin{document}

\maketitle

\begin{abstract}
We study the adversarial bandit problem with composite anonymous delayed feedback. In this setting, losses of an action are split into $d$ components, spreading over consecutive rounds after the action is chosen. And in each round, the algorithm observes the aggregation of losses that come from the latest $d$ rounds. Previous works focus on oblivious adversarial setting, while we investigate the harder non-oblivious setting. We show non-oblivious setting incurs $\Omega(T)$ pseudo regret even when the loss sequence is bounded memory. However, we propose a wrapper algorithm which enjoys $o(T)$ policy regret on many adversarial bandit problems with the assumption that the loss sequence is bounded memory. Especially, for $K$-armed bandit and bandit convex optimization, we have $\mathcal{O}(T^{2/3})$ policy regret bound. We also prove a matching lower bound for $K$-armed bandit. Our lower bound works even when the loss sequence is oblivious but the delay is non-oblivious. It answers the open problem proposed in \cite{wang2021adaptive}, showing that non-oblivious delay is enough to incur $\tilde{\Omega}(T^{2/3})$ regret.
\end{abstract}

\section{Introduction}
Multi-armed bandit is a widely studied problem. It can be formulated by a multi-rounds game between two players,  an adversary and a learner. In $t$-th round, the adversary assigns each action $a\in [K]$ a loss $l_t(a)$, and simultaneously, the learner chooses an action $a_t\in [K]$. It then incurs loss $l_t(a_t)$. The learner can observe the loss this round of the action it just chose $\left(\mbox{i.e., }l_t(a_t) \right)$, but can not observe the loss of other actions. This is so-called \emph{bandit feedback}. Learner’s goal is to minimize its total loss, We usually choose a posteriori best fixed action as the comparison, so that the goal is to minimize the \emph{expected regret}, defined as
\[\mathbb{E}[R_T] = \mathbb{E}\left[\sum_{t=1}^T l_t(a_t)-\min_{y \in \mathcal{A}}\sum_{t=1}^T l_t(y)\right]\]

Bandits problem has a wide range of applications in the industry, including medical trials, recommendation systems, computational ads., tree search algorithms..(\cite{kocsis2006discounted,chapelle2014simple,villar2015multi,silver2016mastering,lei2017actor}). In the standard formulation of the bandit problem, each round the learner observes the precise feedback immediately, and adjusts its strategy afterward according to the immediate feedback. However, in many real-world situations, this assumption can not be satisfied. The total impact of an action may not be observed immediately. In contrast, the impact may spread over an interval after the action has been played. For instance, consider the advertising problem. People do not always click on the website or buy the product after seeing the ads immediately, and the feedback(i.e., click number) the recommender observed may be the aggregation of the impact of several ads recommended before. 

To address the above scenarios, \cite{pike2018bandits} proposed a stochastic bandit model with anonymous feedback. In their setting, round $t$ is related with a delay time $d(t)$, which is drawn from some i.i.d distribution, and the feedback learner can observe at the end of round $t$ is the aggregation $\sum_{s+d(s)=t} l_s(a_s)$. They showed that there is a learner achieving $\tilde{\mathcal{O}}\left(\sqrt{KT}+K\mathbb{E}(d)\right)$ expected regret. \cite{cesa2018nonstochastic} generalized the model to adversarial bandits, where the loss in their model is a composition of constant parts, $l_t(a) = \sum_{s=0}^{d-1}l_t^{(s)}(a)$ where $l_t^{(s)}(a)$ means the part of loss which will delay $s$ rounds. The feedback learner observes at the end of $t$-th round is $l_t^o(a_t)=\sum_{s=0}^{d-1}l_{t-s}^{(s)}(a_{t-s})$. In their model, the loss sequence and delay are both oblivious, which means they can not be adjusted according to the learner's action history. They proposed a mini-batch wrapper algorithm, converting a standard bandit algorithm to the one that can handle composite anonymous delayed feedback. They applied this wrapper on EXP3 algorithm\cite{auer2002nonstochastic}, achieving $\tilde{\mathcal{O}}(\sqrt{dKT})$ expected regret of multi-armed bandit problem with composite anonymous feedback. Their wrapper can also be applied on bandit convex optimization problem \cite{flaxman2005online}, in which the action set is a convex body, and loss functions are bounded convex functions on this convex body. They applied their wrapper on the algorithm proposed in \cite{saha2011improved}, achieving $\tilde{\mathcal{O}}(d^{1/3}(KT)^{2/3})$ regret. Subsequently, \cite{wang2021adaptive} studied the situation that delay is determined by a non-oblivious adversary. That is, though the loss sequence is oblivious, an adversary can split the loss into parts according to the learner's action history. In this setting, they modified standard EXP3 algorithm so that it achieves $\tilde{\mathcal{O}}((d+\sqrt{K})T^{2/3})$ regret for $K$-armed bandit. Different from other previous algorithms, it does not require any prior knowledge of delay $d$. Though they believed that their algorithm is asymptotic optimal for the adversary with oblivious loss and non-oblivious delay, we only have a $\Omega(\sqrt{T})$ regret lower bound from the classical multi-armed bandit problem. How to derive a matching regret lower bound is one of the future research problems in their work.

The existing works on composite anonymous feedback setting all assume that the loss sequence is oblivious, which does not always hold in the real world. For example, consider the case that one is involved in a repeated game with other players, it is natural that others will adjust their strategies according to his action history. This will make the loss of each pure strategy non-oblivious. In $K$-armed bandit and bandit convex optimization problem without delay, even if we consider non-oblivious loss, we still have $\tilde{\Theta}(\sqrt{T})$ pseudo regret. However, things become very different when we consider composite anonymous delayed feedback.

\paragraph{Contribution.}We studied the bandits with composite anonymous feedback under \emph{non-oblivious} setting. In our model, we allow both non-oblivious delay and loss sequences. Since when the loss sequence is non-oblivious, the common regret can be generalized to different performance metrics. We first discuss which performance metric to use in our setting. We show that any learner can not achieve sublinear \emph{external pseudo regret} under our non-oblivious setting. Inspired by \cite{arora2012online}, we turn to a more reasonable metric called \emph{policy regret}. In non-oblivious setting, policy regret is $\Omega(T)$ even without delayed feedback. So \cite{arora2012online} considered a weaker adversary which has \emph{bounded memory}. They proved $o(T)$ policy regret bounds for many bandit problems with bounded memory assumption. Especially, they proved $\mathcal{O}(T^{2/3})$ policy regret bounds for $K$-armed bandit. Different from the pseudo regret metric, we find that the policy regret does not get worse with the introduction of delay. We prove that the simple mini-batch wrapper can generate $o(T)$ \emph{policy regret} algorithms for many bandit problems with non-oblivious delays and \emph{bounded memory} loss sequence in composite anonymous feedback setting. Especially, it can generate $\mathcal{O}(T^{2/3})$ policy regret algorithms for $K$-armed bandits and bandit convex optimization in our setting. Moreover, this mini-batch wrapper does not require any prior knowledge of $d$. Meanwhile, pseudo regret is still $\Omega(T)$ even when we restrict the loss sequence to be bounded memory. Since policy regret is the same as common regret in the oblivious setting, our upper bound can be seen as a generalization of the result in \cite{wang2021adaptive}. Furthermore, we prove a matching lower bound for our model. In fact, we prove a stronger lower bound. Even if the loss sequence is generated by an oblivious adversary, any learner can only obtain $\tilde{\Omega}(T^{2/3})$ regret. Our lower bound answered the problem proposed in \cite{wang2021adaptive}, showing that non-oblivious delay on its own is enough to cause $\tilde{\Theta}(T^{2/3})$ regret.

To summarize the above results, our study provides a complete answer to the non-oblivious composite anonymous feedback setting.

\paragraph{More Related Work.}Delay setting was first considered in \cite{gergely2010online}, in which they assumed each feedback is delayed by a constant $d$, and they posed a $\tilde{\mathcal{O}}(\sqrt{dKT})$ regret algorithm. \cite{joulani2013online} generalized this result to the partial monitoring setting. \cite{quanrud2015online} first considered the non-uniform delay setting, where the delay size of each round can be different. Let $D$ be the sum of all delay sizes, they proved a $\mathcal{O}(\sqrt{(D+T)\log K})$ regret bound in \emph{full feedback} online learning setting. In recent years, a series of works have generalized this work to the bandit feedback setting and keep developing more instance-dependent upper bound(\cite{li2019bandit,thune2019nonstochastic,bistritz2019exp3,zimmert2020optimal,gyorgy2021adapting,cella2020stochastic}). All above works assumed non-anonymous and separated feedback. That is, the learner observes every single feedback rather than their aggregation. Besides, \cite{li2019bandit} studied an unknown delay setting where the learner does not know which round the feedback comes from. Their feedback is also separated, and the learner knows which action is related to the feedback.

Policy regret and bounded memory assumption are proposed in \cite{arora2012online}. They are used in many online learning and bandit literature, including \cite{anava2015online,heidari2016tight,arora2018policy,jaghargh2019consistent}.
\section{Model Setting}
Adversarial bandit problem is a repeated game between a learner and an adversary. There are $T$ rounds in this game. In $t$-th round, the learner chooses an action $a_t \in \mathcal{A}$, the adversary chooses a function $l_t \in \mathcal{L}$ mapping $\mathcal{A}$ to a normalized bounded interval $[0,1]$. Then the learner incurs loss $l_t(a_t)$. The learner's target is to minimize its expected cumulative regret.

Here we formulate two classical bandit problems into the instance of adversarial bandit problem.
\begin{example}[$K$-armed bandit]
$K$-armed bandit is a special case of adversarial bandit problem where $\mathcal{A}=[K]\triangleq\{1,2,\cdots,K\}$. And $\mathcal{L}$ is the set of all functions mapping $[K]$ to $[0,1]$.
\end{example}

\begin{example}[bandit convex optimization]
Bandit convex optimization is also a particular case of adversarial bandit problem when $\mathcal{A}$ is a convex body of $\mathbb{R}^K$. And $\mathcal{L}$ contains all $L$-Lipschitz convex function where $L$ is a constant.
\end{example}

\paragraph{Adversary setting.}Our setting allows the adversary to be non-oblivious, which means that it can choose functions $l_t$ according to the action history $a_1,a_2,\cdots,a_{t-1}$ of the learner. To formalize this setting, we can think that $l_t$ takes the whole action histories sequence $A_t=(a_1,a_2,\cdots,a_t)$ as its input. Under this viewpoint, we can assume that $l_t$ is determined before the game starts. And $l_t(A_t)\in \mathcal{L}$ means if we fix $a_1,\cdots,a_{t-1}$, it belongs to $\mathcal{L}$.

\paragraph{Delay setting.}In this paper, we consider the composite anonymous delayed feedback setting proposed by \cite{cesa2018nonstochastic}. In this setting, the adversary can split the loss function into $d$ components arbitrarily, where $d$ is a constant. That is $l_t(A_t) = \sum_{s=0}^{d-1}l_t^{(s)}(A_t)$. We also allow this splitting process to be non-oblivious, which means $l_t^{(s)}$ can be chosen according to the action histories $A_t$. At the end of $t$-th round, after the algorithm has made its decision this round, the algorithm can only observe $l_t^o(A_t)=\sum_{s=0}^{d-1}l_{t-s}^{(s)}(A_{t-s})$, but can not figure out how $l_t^o(A_t)$ is composed.

\paragraph{Pseudo-regret and policy regret.}In non-oblivious setting, the most common metric of the performance of a learner is \textit{external pseudo-regret}. Defined as 
\begin{equation}
    R_T^{pseudo}=\sum_{t=1}^T l_t(A_t)-\min_{y\in \mathcal{A}} \sum_{t=1}^T l_t(A_{t-1},y)
\end{equation}

Though \textit{external pseudo-regret} is widely used, its meaning is quite strange for the non-oblivious setting because if the learner actually chooses $y$ every round, the loss he gets is not $\sum_{t=1}^T l_t(A_{t-1},y)$ but $\sum_{t=1}^T l_t(y,y,\cdots,y)$. This fact inspires people to design more meaningful metrics in non-oblivious setting. \cite{arora2012online} proposed a new metric called \textit{policy regret} which is defined as
\begin{equation}
    R_T^{policy}=\sum_{t=1}^T l_t(A_t)-\min_{y\in \mathcal{A}} \sum_{t=1}^T l_t(y,\cdots,y)
\end{equation}
Policy regret captures the fact that learner's different action sequences will cause different loss sequences. We believe that policy regret is a more reasonable metric compared with external pseudo regret. Both pseudo-regret and policy regret are the same as standard regret definition when the loss is oblivious.

\cite{arora2012online} shows that policy regret has a $\Omega(T)$ lower bound against the non-oblivious adversary in $K$-armed bandits without delayed feedback. So they consider a weaker adversary which has bounded memory.

\begin{definition}[$m$-bounded memory]
    A non-oblivious loss sequence $l_t$ is called \emph{$m$-bounded memory} if for any action sequence $a_1, a_2, \cdots, a_t$ and $a_1',a_2',\cdots, a_{t-m-1}'$, the following holds
    \[l_t(a_1,a_2,\cdots,a_t)=l_t(a_1',a_2',\cdots,a_{t-m-1}',a_{t-m},\cdots,a_t)\] holds. We call a non-oblivious loss sequence is \emph{bounded memory} iff $m$ is a constant.
\end{definition}

From the definition, $0$-bounded memory loss sequence is an oblivious loss sequence. \cite{arora2012online} proved a $\tilde{\mathcal{O}}(K^{1/3}T^{2/3})$ policy regret upper bound for no delay setting under the assumption that the adversary is bounded memory. In this paper, we also assume that the loss sequence is bounded memory. However, we do not restrict the memory of delay adversary. As an example, the model in \cite{wang2021adaptive} can be seen as the same as our model with $0$-bounded memory loss sequence.

In our non-oblivious composite anonymous feedback setting, external pseudo-regret has a $\Omega(T)$ lower bound even when the loss sequence is generated by an $1$-bounded memory adversary, which means this setting is not learnable under the external pseudo-regret metric. Formally, we prove the following theorem.

\begin{theorem}
    \label{thm:thm1}
    In composite anonymous feedback setting, there is a $2$-armed bandit with non-oblivious adversary, such that any learner incurs $\Omega(T)$ expected pseudo-regret. Moreover, the loss sequence is $1$-bounded memory.
\end{theorem}
\begin{proof}
    Let action set $\mathcal{A}=\{1,2\}$. By Yao's minmax principle, it is enough to construct a distribution over adversaries such that any deterministic learner can only achieve $\Omega(T)$ expected pseudo-regret. To do this, we will construct an adversary so that no learner can gain any information from its observation. The delay adversary is simple, at odd rounds, it delay all loss to the next round, at even rounds, it does not delay the loss. To be formally, $l_t^{(0)}(A_t)=0 ,l_t^{(1)}(A_t)=l_t(A_t)$ when $t$ is odd, $l_t^{(0)}(A_t)=l_t(A_t) ,l_t^{(1)}(A_t)=0$ when $t$ is even.

    The loss $l_t(A_t)$ is related to a random variable $Z$, where $Z$ is picked up from $\mathcal{A}$ uniformly at random. At the odd round, the adversary set the loss of arm $Z$ to be $0$ and the loss of another arm to be $1$. At the even round, if the learner chose arm $Z$ at the previous round, the adversary set losses of both arms to be $1$. If the learner did not choose arm $Z$ in the previous round, the adversary set losses of both arms to be $0$.

    Obviously, the observation sequence of any deterministic learner is $0,1,0,1,0,1,\cdots$ regardless of their actions per round, so it can not distinguish between different $Z$. Let $N_i^{odd}$ be the times the learner chose arm $i$ at odd rounds. Then $\mathbb{E}[N_{i}^{odd}|Z=1]=\mathbb{E}[N_{i}^{odd}|Z=2]$ since the learner can not distinguish $Z$. This leads to
    \begin{align*}
            \mathbb{E}[R_T^{pseudo}]&\geq \frac{1}{2} \mathbb{E}[R_T^{pseudo}|Z=1]+ \frac{1}{2} \mathbb{E}[R_T^{pseudo}|Z=2]\\
            &\geq \frac{1}{2}\mathbb{E}[N_{2}^{odd}|Z=1]+\frac{1}{2}\mathbb{E}[N_{1}^{odd}|Z=2]\\
            &=\frac{1}{2}\mathbb{E}[N_{1}^{odd}+N_{2}^{odd}|Z=1]=\frac{T}{4}
    \end{align*}
\end{proof}

As we can see from the next chapter, though our setting is not learnable under the pseudo-regret metric, it is learnable under the policy regret metric with the assumption that loss sequence is bounded memory.

\section{Upper Bound}
In this section, we prove that by applying a mini-batch wrapper, one can convert any standard non-oblivious bandit algorithm to an algorithm that can handle non-oblivious adversary with composite anonymous delayed feedback.

Intuitively, when the learner chooses an action different from the last round, the feedback it observes in the interval of $d$ rounds after that can not reflect the true losses of the actions it chooses. In other words, the learner suffers from observing inaccurate feedback during a $d$ rounds interval immediately when it switches its chosen action. To obtain accurate feedback for decisions, a learner can not switch its chosen action frequently. A natural approach is to apply a mini-batch wrapper on a standard bandit algorithm. That is, we divide all $T$ rounds into $\lceil T/\tau \rceil$ batches where $\tau$ is the batch size. Each batch contains $\tau $ consecutive rounds except the last round, which may contain fewer than $\tau $ rounds. At the beginning of the $j$-th batch, we receive an action $z_j\in \mathcal{A}$ from the black-box algorithm and keep choosing $z_j$ during this batch. At the end of the $j$-th batch, we feed the average loss observed in this batch to the black-box algorithm. Since the black-box algorithm can only receive a $[0,1]$ loss, we feed the minimal between average loss and $1$. See \algo{alg} for the pseudo-code.

\begin{algorithm}[tb]
	\caption{Mini-batch wrapper} 
	\label{algo:alg}
	\begin{algorithmic}[1]
		\Require Black-box bandit algorithm $\mathcal{B}$, batch size $\tau$, time horizon $T$
		\State $j\gets 1$
		\While{not end}
		    \If{there are less than $\tau$ rounds}
		        \State choose arbitrary action in remaining rounds.
		    \Else
		        \State query $\mathcal{B}$ for the next action $z_j\in\mathcal{A}$
		        \State choose $z_j$ for consecutive $\tau$ rounds, collect feedback $l_t^o$ for $t\in [(j-1)\tau+1,j\tau]$
		        \State feed $\mathcal{B}$ with $\hat{l}_j=\min \left\{\frac{1}{\tau}\sum_{t=(j-1)\tau+1}^{j\tau}l_t^o,1\right\}$ as the feedback of action $z_j$
		        \State $j\gets j+1$
		    \EndIf
		\EndWhile
	\end{algorithmic} 
\end{algorithm}

By applying the mini-batch wrapper, times of the learner's action switching can be controlled by $\mathcal{O}(T/\tau)$. In each batch, the learner suffers the inaccurate feedback for constant rounds so that it only curses a constant feedback error, which will add to the regret finally. If we set the batch size to be a relatively large quantity such that $T/\tau =o(T)$, it is possible that we control the regret to $o(T)$. 
\begin{theorem}
    \label{thm:thm2}
    Suppose we have a bandit algorithm $\mathcal{B}$ for no delay setting which achieves $R(J)$ expected pseudo regret when the time horizon is $J$ and the loss sequence is generated by a non-oblivious adversary. Assume $\tau>\max\{d,m\}$ and the loss sequence is $m$-bounded, then the mini-batch wrapper (\algo{alg}) can achieve policy regret as follows for the composite anonymous feedback setting. 
    \[\mathbb{E}[R_{T}^{policy}]\leq\tau R(\lfloor T/\tau\rfloor)+\mathcal{O}(\max\{m,d\} T/\tau)+\mathcal{O}(\tau)\]
\end{theorem}
\begin{proof}
    Without loss of generality, we can assume $T/\tau$ is an integer, otherwise it only produces an extra $\mathcal{O}(\tau)$ term in the expected policy regret bound. Since $z_{j}$ is the action of $j$-th batch, we have $a_{(j-1)\tau+1}=a_{(j-1)\tau+2}=\cdots=a_{j\tau}=z_{j}$. Due to the pseudo regret bound of the black-box algorithm, we have
    \begin{align*}
        \mathbb{E}\left[\max_{y\in\mathcal{A}}\sum_{j=1}^{T/\tau}(\hat{l}_{j}(Z_j)-\hat{l}_{j}(Z_{j-1},y))\right]\leq R(T/\tau )
    \end{align*}
    Denote $y^t$ be the $t$-repetition sequence $(y,y,\cdots,y)$. To bound $R_T^{policy}(y)=\sum_{t=1}^Tl_t(A_t)-\sum_{t=1}^{T}l_t(y^t)$ , we decompose it into following terms
    \begin{align}
            \begin{aligned}\label{eq:main}
            \sum_{t=1}^T l_t(A_t)-\sum_{t=1}^{T}l_t(y^t)=&\left(\sum_{t=1}^{T}l_t(A_t)-\sum_{t=1}^{T}l_t^o(A_t)\right)+\sum_{j=1}^{T/\tau}\left(\sum_{t=(j-1)\tau +1}^{j\tau}l_t^o(A_t)-\tau\hat{l}_j(Z_j) \right)\\&+\tau \sum_{j=1}^{T/\tau}\left(\hat{l}_{j}(Z_j)-\hat{l}_{j}(Z_{j-1},y)\right)+\left(\tau\sum_{j=1}^{T/\tau}\hat{l}_t(Z_{j-1},y)-\sum_{t=1}^Tl_t(y^t)\right)
            \end{aligned}
    \end{align}
    We bound above formula term by term. To make the proof more concise, we let $l_t(A_t)=0 \ ,\forall t\leq 0$. And we admit the convention that $\sum_{i=a}^b(\mbox{whatever})=0 $ if $a>b$. For $\sum_{t=1}^Tl_t(A_t)-\sum_{t=1}^{T}l_t^o(A_t)$, it is the difference between true loss and observed feedback. Since only the loss of the last $d-1$ rounds may not be fully observed, this difference could be bounded by $d-1$ as follows
    \begin{align*}
            \sum_{t=1}^Tl_t(A_t)-\sum_{t=1}^{T}l_t^o(A_t)&=\sum_{t=1}^Tl_t(A_t)-\sum_{t=1}^{T}\sum_{s=0}^{d-1}l_{t-s}^{(s)}(A_{t-s})\\&=\sum_{t=1}^T\sum_{s=0}^{d-1}l_t^{(s)}(A_t)-\sum_{t=1}^{T}\sum_{s=0}^{\min\{d-1,T-t\}}l_{t}^{(s)}(A_{t})\\&=\sum_{t=T-d+2}^T \sum_{s=T-t+1}^{d-1}l_t^{(s)}(A_t)\\&\leq d-1
    \end{align*}
    The last equality is because 
    \begin{align*}
        \sum_{s=T-t+1}^{d-1}l_t^{(s)}(A_t)\leq \sum_{s=0}^{d-1}l_t^{(s)}(A_t)=l_t(A_t)\leq 1.
    \end{align*}

    For the second term,
    \begin{align*}
        \sum_{t=(j-1)\tau +1}^{j\tau}l_t^o(A_t) &= \sum_{t=(j-1)\tau +1}^{j\tau}\sum_{s=0}^{d-1}l_{t-s}^{(s)}(A_{t-s})\\&= \sum_{t=(j-1)\tau -d+2}^{j\tau}\sum_{s=\max \{0,(j-1)\tau +1-t\}}^{\min \{d-1,j\tau-t\}}l_t^{(s)}(A_t)\\&\leq \tau+d-1
    \end{align*}
    According to the definition of $\hat{l}_t (A_t)$, 
    \begin{align*}
        \sum_{t=(j-1)\tau +1}^{j\tau}l_t^o(A_t)-\tau\hat{l}_j(Z_j)= \max \left\{0,\sum_{t=(j-1)\tau +1}^{j\tau}l_t^o(A_t) -\tau \right\}\leq d-1
    \end{align*}
        And we will bound the third term by the expected regret bound of the black box algorithm. For the last term, we further decompose it into
        \begin{align}
        \begin{aligned}
        \label{eq:split2}
                \tau\sum_{j=1}^{T/\tau}\hat{l}_t(Z_{j-1},y)-\sum_{t=1}^Tl_t(y^t)=&\left(\tau\sum_{j=1}^{T/\tau}\hat{l}_t(Z_{j-1},y)-\sum_{j=1}^{T/\tau }\sum_{t=(j-1)\tau+1}^{j\tau}l_t^o(A_{(j-1)\tau},y^{t-(j-1)\tau}) \right)\\ &+\left(\sum_{j=1}^{T/\tau }\sum_{t=(j-1)\tau+1}^{j\tau}l_t^o(A_{(j-1)\tau},y^{t-(j-1)\tau})-\sum_{j=1}^{T/\tau }\sum_{t=(j-1)\tau+1}^{j\tau}l_t(A_{(j-1)\tau},y^{t-(j-1)\tau}) \right)\\&+\left(\sum_{j=1}^{T/\tau }\sum_{t=(j-1)\tau+1}^{j\tau}l_t(A_{(j-1)\tau},y^{t-(j-1)\tau})-\sum_{t=1}^T l_t(y^t)\right)
        \end{aligned}
        \end{align}
        Note that by the definition of $\hat{l}_t (A_t)$
        \[\tau \hat{l}_j(Z_{j-1},y)=\min \left\{\sum_{j=1}^{T/\tau }\sum_{t=(j-1)\tau+1}^{j\tau}l_t^o(A_{(j-1)\tau},y^{t-(j-1)\tau}),\tau\right\}\]
        Thus, the first term of \eq{split2} is non-positive. And we define $j_t\triangleq\lceil t/\tau \rceil$, which means the batch of the $t$-th round. Then
        \begin{align*}
                    \sum_{j=1}^{T/\tau}\sum_{t=(j-1)\tau+1}^{j\tau}l_t^o(A_{(j-1)\tau},y^{t-(j-1)\tau})&=\sum_{j=1}^{T/\tau}\sum_{t=(j-1)\tau+1}^{j\tau}\left(\sum_{s=0}^{t-(j-1)\tau-1}l_{t-s}^{(s)}(A_{(j-1)\tau},y^{t-s-(j-1)\tau})\right.\\&\quad \left.+\sum_{s=t-(j-1)\tau}^{d-1}l_{t-s}^{(s)}(A_{t-s})\right)\\
                    &=\sum_{t=1}^{T}\left(\sum_{s=0}^{t-(j_{t}-1)\tau-1} l_{t-s}^{(s)}(A_{(j_t-1)\tau},y^{t-s-(j_t-1)\tau})+\sum_{s=t-(j_{t}-1)\tau}^{d-1}l_{t-s}^{(s)}(A_{t-s})\right)\\
                    &\leq \sum_{t=1}^T\left(\sum_{s=0}^{d-1}\left(\mathbb{I}\left[t>(j_{t+s}-1)\tau+1\right]\cdot l_{t}^{(s)}(A_{(j_{t}-1)\tau},y^{t-(j_{t}-1)\tau})\right.\right.\\&\quad\left.\left.+\mathbb{I}[t\leq(j_{t+s}-1)\tau]\cdot l_{t}^{(s)}(A_{t})\right)\right)
        \end{align*}
        Thus, the second term of \eq{split2} is bounded by the following inequality:
        \begin{align*}
                    \mbox{second term of \eq{split2}}&\leq \sum_{t=1}^{T}\sum_{s=0}^{d-1}\mathbb{I}[t\leq(j_{t+s}-1)\tau]\cdot\left(l_{t}^{(s)}(A_{t})-l_{t}^{(s)}(A_{(j_{t}-1)\tau},y^{t-(j_{t}-1)\tau})\right)\\&\leq \sum_{t=1}^{T}\mathbb{I}[t\leq(j_{t+d-1}-1)\tau]\\&\leq \frac{T}{\tau}(d-1)
        \end{align*}
The second inequality is because if $t>(j_{t+d-1}-1)\tau$, then 
\begin{align*}
\sum_{s=0}^{d-1}\mathbb{I}[t\leq(j_{t+s}-1)\tau]\left(l_{t}^{(s)}(A_{t})-l_{t}^{(s)}(A_{(j_{t}-1)\tau},y^{t-(j_{t}-1)\tau})\right)=0
\end{align*}
Else, since 
\begin{align*}
    \sum_{s=0}^{d-1}\mathbb{I}[t\leq(j_{t+s}-1)\tau]\cdot l_{t}^{(s)}(A_{t})\leq l_t(A_t)\leq 1
\end{align*} and 
\begin{align*}
    \sum_{s=0}^{d-1}\mathbb{I}[t\leq(j_{t+s}-1)\tau]\cdot l_{t}^{(s)}(A_{(j_{t}-1)\tau},y^{t-(j_{t}-1)\tau})\leq l_{t}(A_{(j_{t}-1)\tau},y^{t-(j_{t}-1)\tau})\leq 1
\end{align*}
we can bound 
\begin{align*}
\sum_{s=0}^{d-1}\mathbb{I}[t\leq(j_{t+s}-1)\tau]\left(l_{t}^{(s)}(A_{t})-l_{t}^{(s)}(A_{(j_{t}-1)\tau},y^{t-(j_{t}-1)\tau})\right)\leq  1
\end{align*}
Now we bound the third term of \eq{split2}:
     \begin{align*}
                    \mbox{third term of \eq{split2}}&= \sum_{j=1}^{T/\tau }\sum_{t=(j-1)\tau+1}^{(j-1)\tau+m}\left(l_t(A_{(j-1)\tau},y^{t-(j-1)\tau})-l_t(y^{t})\right)\\&\quad+\sum_{j=1}^{T/\tau}\sum_{t=(j-1)\tau+m+1}^{j\tau}\left(l_t(A_{(j-1)\tau},y^{t-(j-1)\tau})-l_t(y^{t})\right)\\&=\sum_{j=1}^{T/\tau }\sum_{t=(j-1)\tau+1}^{(j-1)\tau+m}\left(l_t(A_{(j-1)\tau},y^{t-(j-1)\tau})-l_t(y^{t})\right)\\& \leq  \sum_{j=1}^{T/\tau}\sum_{t=(j-1)\tau+1}^{(j-1)\tau+m} 1=\frac{mT}{\tau} 
        \end{align*}

    When we plug all above bounds into \eq{main}, $\forall y\in \mathcal{A}$, the following holds
    \begin{align*}
        \sum_{t=1}^Tl_t(A_t)-\sum_{t=1}^{T}l_t(y^t)\leq (d-1) + (d-1)\frac{T}{\tau} +\tau \sum_{j=1}^{T/\tau}\left(\hat{l}_{j}(Z_j)-\hat{l}_{j}(Z_{j-1},y)\right) + (d-1)\frac{T}{\tau}+\frac{mT}{\tau}
    \end{align*}
    Maximize over $y\in \mathcal{A}$ and take expectation both sides, we have
    \begin{align*}
        \mathbb{E}\left[\max_{y\in \mathcal{A}}\left(\sum_{t=1}^Tl_t(A_t)-\sum_{t=1}^{T}l_t(y^t)\right)\right]\leq \
        \tau R(T/\tau)+(2d+m-2)\frac{T}{\tau}+d-1
    \end{align*}
\end{proof}
By applying the mini-batch wrapper on some bandit algorithms, we obtain algorithms that can handle composite anonymous delayed feedback setting. Firstly, we apply \thm{thm2} on $K$-armed bandit problem. \cite{auer2002nonstochastic} proposed a well known algorithm $EXP3$ which guarantees $\tilde{\mathcal{O}}(\sqrt{K T})$ expected pseudo regret. Employing this algorithm, we have the following corollary.

\begin{corollary} 
  For $K$-armed bandit problem, if we apply \algo{alg} on $EXP3$ algorithm, and set batch size $\tau = K^{-1/3}T^{1/3}$, we have the expected policy regret satisfies
  \[\mathbb{E}[R_T^{policy}]\leq \tilde{\mathcal{O}}(\max\{m,d\} K^{1/3}T^{2/3})\]

\end{corollary}

For bandit convex optimization, \cite{kernel17} proposed an algorithm using kernel method, and it can guarantee $\tilde{\mathcal{O}}(K^{9.5}\sqrt{T})$ expected pseudo regret, where $K$ is the dimension of the action space $\mathcal{A}$. By employing this algorithm as our black-box algorithm $\mathcal{B}$, we have the following corollary.
\begin{corollary}
    For bandit convex optimization, if we apply \algo{alg} on the algorithm in \cite{kernel17}, and set batch size $\tau = K^{-19/3}T^{1/3}$, we have the expected policy regret satisfies
  \[\mathbb{E}[R_T^{policy}]\leq \tilde{\mathcal{O}}(\max\{m,d\} K^{19/3}T^{2/3})\]

\end{corollary}

Other algorithms for composite anonymous feedback setting are also some kinds of mini-batch wrapper, such as CLW in \cite{cesa2018nonstochastic} and ARS-EXP3 in \cite{wang2021adaptive}. ARS-EXP3 uses increasing batch sizes on EXP3 algorithm. CLW is a wrapper algorithm that can be applied on many normal bandit algorithms, and it uses a constant batch size. As we discussed above, every action switching may incur extra constant regret. At first, it seems that CLW will incur $\Omega(T)$ regret since CLW performs $\Omega(T)$ action switches. However, thanks to the oblivious adversary they assumes, CLW can randomize the batch size to fool the oblivious adversary, such that they can still reach $\mathcal{O}(\sqrt{T})$ regret. Nonetheless, this randomizing technique does not work when delay is non-oblivious since the adversary here can read learner's random bits realized before the current round.

\section{Lower Bound}
In this section, we prove a matching lower bound for $K$-armed bandit. The main result of \cite{dekel2014bandits} actually implies a $\tilde{\Omega}(T^{2/3})$ policy regret lower bound for $K$-armed bandit without delayed feedback. However, their lower bound depends on constructing a non-oblivious loss sequence. Our lower bound here is stronger since the loss sequence in our construction is oblivious. The lower bound shows that the non-oblivious delay adversary on its own is enough to incur $\tilde{\Theta}(T^{2/3})$ regret without the help of non-oblivious loss.

\begin{theorem}
\label{thm:thm3}
For $K$-arms bandit with non-oblivious delays and oblivious loss sequence, we have
\[\mathbb{E}[R_T] =\tilde{\Omega}(K^{1/3}T^{2/3})\]
\end{theorem}

Note that we use the notation $R_T$ rather than $R_T^{policy}$, since in \thm{thm3}, the loss sequence is oblivious, and the policy regret is the same as normal regret.

According to Yao's minimax principle, to prove \thm{thm3}, it is enough to construct a distribution over some loss sequences and a deterministic delay adversary, such that any deterministic learner can only achieve $\tilde{\Omega}(K^{1/3}T^{2/3})$ expected regret. Before we discuss our detailed construction, we firstly describe our intuition briefly. We construct the loss sequence based on a random walk. The best arm always has the loss $\epsilon$ lower than the random walk, while others always have the loss equal to the random walk. This construction can force the learner to switch between arms since the learner only observes a random walk and obtain no information if he keeps choosing one arm. The delay adversary we construct makes the observed losses in one round are the same no matter which arm is chosen in this round. Therefore, the learner can only get information from the changes of observed loss between rounds. Our delay construction makes the changes of the observed loss as consistent as possible with the changes of the random walk. So it happens only a few times that the observed loss sequence deviates from the random walk. We bound the information learner can get in this deviation through a KL divergence argument. Thus, the total information learner get is so low that it can not help the learner achieve low regret. However, a random walk can drift so much that it jumps out of $[0,1]$, making the construction of the loss sequence invalid. To address this problem, we borrow the idea of multi-scale random walk from \cite{dekel2014bandits}. Multi-scale random walk is a trade-off between random walk and i.i.d samples. Its drift can be bounded in an acceptable range while still maintaining the low information nature of the random walk. 

To clarify our constructional proof, we describe the construction of the loss sequence in \secc{41}, and the construction of the delay adversary is in \secc{42}. We sketch the proof of \thm{thm3} in \secc{proof}. Besides, we introduce the some information lower bound tools in \secc{tools}.

\subsection{Information Theoretic Tools}
\label{sec:tools}
Given two probability measures $\mathbb{P}$ and $\mathbb{Q}$ over real field $\mathbb{R}$, the distance between these two measures can be captured by their \emph{Kullback-Leibler Divergence}. 

\begin{definition}[Kullback-Leibler divergence]
Let $p,q$ be the density function of $\mathbb{P},\mathbb{Q}$ with respect to a dominating measure $\lambda$, then the Kullback-Leibler divergence from $\mathbb{Q}$ to $\mathbb{P}$ is 
\begin{align*}
    \mathcal{D}_{KL} (\mathbb{P}\| \mathbb{Q})=\int_{\mathbb{R}} p(x)\log \frac{p(x)}{q(x)} d\lambda (x)
\end{align*}
\end{definition}

We point out that the definition of Kullback-Leibler divergence does not depend on the choice of dominating measure. 

Another distance metric of two probability measures is \emph{total variation}, also known as $L_1$ distance. 

\begin{definition}[total variation]
Let $\mathcal{F}$ be a sigma field of $\mathbb{R}$, and $\mathbb{P},\mathbb{Q}$ are two probability measures over $\mathcal{F}$. The total variation between them is
\begin{align*}
    \mathcal{D}_{TV}^{\mathcal{F}} (\mathbb{P},\mathbb{Q})=\sup_{F\in \mathcal{F}}|\mathbb{P}(F)-\mathbb{Q}(F)|
\end{align*}
\end{definition}

Total variation can be upper bounded in terms of Kullback-Leibler divergence using \emph{Pinsker's Inequality}.

\begin{lemma}[Pinsker's inequality~\cite{cover1999elements}]
Let $\mathcal{F}$ be a sigma field of $\mathbb{R}$, and $\mathbb{P},\mathbb{Q}$ are two probability measures over $\mathcal{F}$. Then
\begin{align*}
    \mathcal{D}_{TV}^{\mathcal{F}}(\mathbb{P},\mathbb{Q})\leq \sqrt{\frac{1}{2}\mathcal{D}_{KL}(\mathbb{P}\| \mathbb{Q})}
\end{align*}
\end{lemma}

\subsection{Construction of Loss Sequence}
\label{sec:41}
In this section, we describe a stochastic process called \emph{multi-scale random walk} proposed in \cite{dekel2014bandits}. It will be used to generate the loss sequence.

Let $\{\xi_t\}$ be i.i.d samples which obey Gaussian distribution $\mathcal{N}(0,\sigma^2)$, where $\sigma^2$ is the variance to be determined later. Let $\rho: [T]\rightarrow \{0\}\cup [T]$ be the \emph{parent function}. The function $\rho$ assigns each round $t$ a parent $\rho(t)$. We restrict that $\rho(t)<t$, and define
\[\begin{split}
    W_t &= W_{\rho(t)} + \xi_t \\ W_0&=0
\end{split}\]
Then $W_t$ is a stochastic process. We next define the width of this stochastic process.

\begin{definition}[cut and width]
The \emph{cut} of a parent function $\rho$ at $t$ is \[\mbox{cut}_{\rho}(t) = \{s\in[T]| \rho(s)\leq t<s\}\] width of $\rho$ is $\omega(\rho) = \max_{t\in [T]}|\mbox{cut}_{\rho}(t)|$
\end{definition}

We then introduce a stochastic process called \emph{multi-scale random walk}.

\begin{definition}[multi-scale random walk]
 Let parent function be $\rho^*(t)= t-2^{\delta_t}$  where $\delta(t)=\max\{i\geq 0| 2^i \mbox{  divides  } t\}$. Then the stochastic process $W_t$ equipped with parent function $\rho^*$ is called multi-scale random walk.
\end{definition}

Multi-scale random walk has the special property that it is not too wide and its drift range is not too large. This property is formalized in following lemmas.

\begin{lemma}[Lemma 2 from \cite{dekel2014bandits}]
The width of multi-scale random walk is bounded by $\lfloor \log_2 T\rfloor+1$.
\end{lemma}

\begin{lemma}[Lemma 1 from \cite{dekel2014bandits}]
\label{lem:range}
Let $W_t$ be the multi-scale random walk, then $\forall \delta\in (0,1)$
\[\mathbb{P}\left(\max_{t\in [T]}  |W_t|\leq \sigma \sqrt{2(\log T+1 )\log \frac{T}{\delta}}\right)\geq 1-\delta \]
\end{lemma}

To specify the loss sequence of each arm, we define a uniform random variable $Z$ valued on $\{0\}\cup \mathcal{A}$. $Z$ tells us which arm is the best. When $Z=0$, all arms have the same loss all the time. We define untruncated loss of arm $a\in \mathcal{A}$ be $l_t'(a)=W_t +\frac{3}{4}-\epsilon \cdot \mathbb{I}[Z=a]$. Where $\epsilon$ is also a parameter to be determined later and $W_t$ is the multi-scale random walk defined above. From now on, we will always use $W_t$ to represent the multi-scale random walk. 

However, this loss may jump out of the bounded interval $[0,1]$, so we truncate it to make sure that loss is in $[0,1]$. For some technical reasons, we further require that loss is greater than $\frac{1}{2}$. The true loss sequence is $l_t(a) = \mbox{trunc}_{[\frac{1}{2},1]}\left(l_t'(a)\right)$
where
\[
\mbox{trunc}_{[a,b]}(x)\boldsymbol{=}\begin{cases}
a & x<a\\
b & x>b\\
x & \mbox{otherwise}
\end{cases}\]
\subsection{Construction of Delay Adversary}
\label{sec:42}
In this section, we describe our construction of delay adversary. The delay adversary is constructed so that it can mislead the learner. To achieve this goal, the delay adversary has two states: \emph{low loss state} and \emph{high loss state}. For $t$-th round, if delay adversary is at low loss state, it splits $l_t(a)=l_t^{(0)}(a)+l_t^{(1)}(a)$ for each arm $a\in \mathcal{A}$, so that 
\begin{equation}\label{eq:cons1}
l_t^{(0)}(a)=\mbox{trunc}_{[\frac{1}{2},1]}(W_{t}+\frac{3}{4}-\epsilon)-l_{t-1}^{(1)}(a_{t-1})\end{equation} If delay adversary is at high loss state, it splits $l_t(a)=l_t^{(0)}(a)+l_t^{(1)}(a)$ for each arm $a\in \mathcal{A}$, so that 
\begin{equation}\label{eq:cons2}
    l_t^{(0)}(a)=\mbox{trunc}_{[\frac{1}{2},1]}(W_{t}+\frac{3}{4})-l_{t-1}^{(1)}(a_{t-1})
\end{equation}

However, sometimes $l_t^{(0)}(a)$ computed by equation \eq{cons1} and equation \eq{cons2} might not lie in $[0, l_t(a)]$, which means the splitting is not valid.  For example, if $l_{t-1}^{(1)}(a_{t-1})=0$, and the delay adversary is at high loss state in round $t$, then consider the best arm $Z$, $l_t(Z)=\mbox{trunc}_{[\frac{1}{2},1]}(W_t+\frac{3}{4}-\epsilon)$ is strictly less than $\mbox{trunc}_{[\frac{1}{2},1]}(W_{t}+\frac{3}{4})$ when $|W_t|< \frac{1}{4}$, so $l_t^{(0)}(Z)$ computed according to \eq{cons2} is greater than $l_t(Z)$, which is invalid. Similar situation happens when delay adversary is at low loss state and $l_{t-1}^{(1)}(a_{t-1}) > W_t+\frac{3}{4}-\epsilon$, $|W_t|< \frac{1}{4}-\epsilon$.

In order to avoid the above situation, the delay adversary will use the following procedure to switch its loss state wisely. Let $S_t$ be the state of delay adversary in round $t$. In round $t$, before splitting the true loss, delay adversary checks $l_{t-1}^{(1)}(a_{t-1})$, the loss delayed from the previous round. If $l_{t-1}^{(1)}(a_{t-1})< \epsilon$ and $S_{t-1}=\mbox{high loss}$, switch the state to low loss, i.e. $S_t=\mbox{low loss}$. If  $l_{t-1}^{(1)}(a_{t-1})> \frac{1}{4}-\epsilon$ and $S_{t-1}=\mbox{low loss}$, switch the state to high loss, i.e. $S_t=\mbox{high loss}$. Otherwise, keep the state unchanged, $S_t=S_{t-1}$. This procedure keeps $l_{t}^{(1)}(a)\in [0,1/4]$ for all $a$ and $t$. Remember we applied a $[1/2,1]$ truncation on the true loss, this lead to $l_t^{(1)}(a)\leq 1/4\leq 1/2\leq l_t(a)$, thus $l_t^{(1)}\in [0,l_t(a)], \forall t,a$. It shows that this state switching procedure is valid. Moreover, we let the delay adversary start from high loss state.

As we will show in \secc{proof}, the times of state switches of the delay adversary is closely related to the information learner can get. Now we aim to upper bound the times of state switches. Firstly, we describe the trajectory of $l_t^{(1)}(a)$ for each $a\in \mathcal{A}$ in the following lemma.
\begin{lemma}
\label{lem:delay1}
Suppose delay adversary is at low loss state in round $t$, then $l_t^{(1)}(Z)=l_{t-1}^{(1)}(a_{t-1})$ and $ l_{t-1}^{(1)}(a_{t-1})+\epsilon \geq l_{t}^{(1)} (a)\geq l_{t-1}^{(1)}(a_{t-1}),\  \forall a\neq Z $. Suppose delay adversary is at high loss state in round $t$, then $l_t^{(1)}(a)=l_{t-1}^{(1)}(a_{t-1}),\forall a\neq Z$ and $ l_{t-1}^{(1)}(a_{t-1})-\epsilon\leq l_{t}^{(1)} (Z)\leq l_{t-1}^{(1)}(a_{t-1})$.
\end{lemma}
\begin{proof}
When delay adversary is at low loss state,  one need to set $l_t^{(0)}(a)=\mbox{trunc}_{[\frac{1}{2},1]}(W_{t}+\frac{3}{4}-\epsilon)-l_{t-1}^{(1)}(a_{t-1})$, since $l_t(Z)=\mbox{trunc}_{[\frac{1}{2},1]}(W_{t}+\frac{3}{4}-\epsilon)$, $l_t^{(1)}(Z) = l_t(Z)-l_t^{(0)}(Z)=l_{t-1}^{(1)}(a_{t-1})$. For arm $a\neq Z$, $l_t(a) = \mbox{trunc}_{[\frac{1}{2},1]}(W_t +\frac{3}{4})$, then $l_t^{(1)}(a)=l_t(a)-l_t^{(0)}(a)=l_{t-1}^{(1)}(a_{t-1})+\mbox{trunc}_{[\frac{1}{2},1]}(W_t +\frac{3}{4})-\mbox{trunc}_{[\frac{1}{2},1]}(W_t +\frac{3}{4}-\epsilon)\leq l_{t-1}^{(1)}(a_{t-1})+ \epsilon$.

When delay adversary is at high loss state, the proof is similar.
\end{proof}
Then we bound the times of state switches as follow.
\begin{lemma}
\label{lem:switches}
If $Z=i$, delay adversary performs at most $\frac{8\epsilon T_i}{1-8\epsilon}$ state switches, where $T_i$ denotes the number of rounds arm $i$ has been been selected.
\end{lemma}

\begin{proof}
Consider the round $t$ in which delay adversary just switches to high loss state, which means $l_{t-1}^{(1)}(a_{t-1})>\frac{1}{4}-\epsilon$. According to \lem{delay1}, compared to $l_{t-1}^{1}(a_{t-1})$, $l_t^{(1)}(a_t)$ may decrease only when learner selects $Z$, and the value of decreasing is at most $\epsilon$. Since the delay adversary only switches its state from high loss to low loss when $l_t^{(1)}(a_t)$ decreases to less than $\epsilon$, learner must select $Z$ for at least $(\frac{1}{4}-2\epsilon)/\epsilon=\frac{1}{4\epsilon}-2$ rounds to make delay adversary switch its state. Then the number of times to switch to low loss state is at most $T_i/(\frac{1}{4\epsilon}-2)=\frac{4\epsilon T_i}{1-8\epsilon}$. Since the state is alternating, the total number of switches is at most $\frac{8\epsilon T_i}{1-8\epsilon}$.
\end{proof}

\subsection{Proof of \texorpdfstring{\thm{thm3}}{}}
\label{sec:proof}
The choice of a deterministic learner in $t$-th round is determined by the observed loss sequence $l_1^o(a_1),\cdots,l_{t-1}^o(a_{t-1})$ before $t$-th round. Let $\mathcal{F}_t$ be the $\sigma$-field generated by $l_1^o(a_1),\cdots,l_{t-1}^o(a_{t-1})$ and $\mathcal{F}\triangleq \mathcal{F}_T$. Let $\mathbb{P}$ be the distribution of observed loss sequence $\{l_t^o(a_t)\}_{t=1}^T$, $\mathbb{P}_t$ be the distribution of $l_t^o(a_t)$ conditioned on the observation history $l_1^o(a_1),\cdots,l_{t-1}^o(a_{t-1})$. Let $\mathbb{P}^i = \mathbb{P}(\cdot| Z=i)$ and $\mathbb{P}_{t}^{i}=\mathbb{P}_{t}(\cdot|\ Z=i)$. Our first step is to bound the 
total variation between $\mathbb{P}^i$ and $\mathbb{P}^0$, denoted by $\mathcal{D}_{TV}^{\mathcal{F}}(\mathbb{P}^0, \mathbb{P}^i)$. We show the total variation is upper bounded by the width of the parent function and the number of times the arm $i$ is chosen, see Lemma 4. Then, intuitively, if the total variation is small, which means that it is hard to distinguish distribution $\mathbb{P}^0$ and $\mathbb{P}^i$, the regret will be large.

Following lemma shows if we truncate two Gaussian measure, their KL-divergence can not become larger. This fact is quite intuitive, since by truncation, we loss some information about the Gaussian measure. We deffer the proof of this lemma to the Appendix.

\begin{lemma}[KL divergence bound between truncated gaussian]
\label{lem:trunc}
Let $\mathbb{P}$ be the $[a,b]$ truncation of gaussian measure with mean $\mu_p$ and variance $\sigma$, and $\mathbb{Q}$ be the $[a,b]$ truncation of gaussian measure with mean $\mu_q$ and variance $\sigma$, then $\mathcal{D}_{KL}(\mathbb{P}\| \mathbb{Q}) \leq \frac{|\mu_p-\mu_q|^2}{2\sigma^2}$
\end{lemma}
With the help of \lem{trunc}, we can prove \lem{TV}.
\begin{lemma}
\label{lem:TV}
$\mathcal{D}_{TV}^\mathcal{F}(\mathbb{P}^0,\mathbb{P}^i)\leq (\epsilon/\sigma)\sqrt{\frac{2\omega(\rho^*)\epsilon \mathbb{E}_{\mathbb{P}^0}[T_i]}{1-8\epsilon}}$, where $\omega(\rho^*)$ is the width of $\rho^*$.
\end{lemma}
\begin{proof}
We calculate KL divergence from $\mathbb{P}^{i}$ to $\mathbb{P}^{0}$ first, by chain rule of KL divergence
\begin{align*}
    \mathcal{D}_{KL}(\mathbb{P}^{0}||\mathbb{P}^{i})=\sum_{t=1}^{T}\mathbb{E}_{\mathbb{P}^0}\left[\mathcal{D}_{KL}(\mathbb{P}_{t}^{0}||\mathbb{P}_{t}^{i})\right]
\end{align*}

For any fixed deterministic learner and condition on any realized observation sequence $l^o_{1}(a_1),l^o_{2}(a_2),$ $l_{3}^o(a_{3}),\cdots,l_{t-1}^o(a_{t-1})$, let $S_t^i$ be the state of delay adversary when $Z=i$. We claim 
\begin{align*}
    \mathcal{D}_{KL}(\mathbb{P}^0_t\|\mathbb{P}^i_t)\leq \frac{\epsilon^{2}}{2\sigma^{2}}\mathbb{I}\{S_{t}^{i}\neq S_{\rho^*(t)}^{i}\}
\end{align*}
If $Z=0$, all arms are high loss, so delay adversary just need to set $l_t^{(1)}(a)=l_t(a)$ to maintain high observed loss. Therefore, $l_t^{(1)}(a_t)$ is always $0$, the delay adversary keeps high loss state unchanged. The observed loss at round $t$ is high loss $l_t^o(a_t)=\mbox{trunc}_{[\frac{1}{2},1]}\left(l_{\rho^*(t)}^o(a_{\rho^*(t)})+\xi_t\right)$.

Now let us consider $\mathbb{P}_t^i$. If $Z=i$ and $S_t^i = S_{\rho^*(t)}^i$, then observed loss at round $t$ is $l_t^o(a_t)=\mbox{trunc}_{[\frac{1}{2},1]}\left(l_{\rho^*(t)}^o(a_{\rho^*(t)})+\xi_t\right)$. $\mathbb{P}^{0}_t,\mathbb{P}^{i}_t$ are truncated gaussian distributions with same variance and mean, therefore $\mathcal{D}_{KL}(\mathbb{P}^{0}_t||\mathbb{P}^{i}_t)=0$. But if $S_{\rho^*(t)}^i = \mbox{high loss}$, $S_{t}^i = \mbox{low loss}$, then 
\[l_t^o(a_t)=\mbox{trunc}_{[\frac{1}{2},1]}\left(l_{\rho^*(t)}^o(a_{\rho^*(t)})+\xi_t-\epsilon\right)\]
If $S_{\rho^*(t)}^i = \mbox{low loss}$, $S_{t}^i = \mbox{high loss}$, then \[l_t^o(a_t)=\mbox{trunc}_{[\frac{1}{2},1]}\left(l_{\rho^*(t)}^o(a_{\rho^*(t)})+\xi_t+\epsilon\right)\]
In this case, $\mathbb{P}_t^{0},\mathbb{P}_t^{i}$ are truncated gaussian distributions with mean differed by $\epsilon$, by \lem{trunc}, $\mathcal{D}_{KL}(\mathbb{P}^{0}_t||\mathbb{P}^{i}_t)\leq \frac{\epsilon^2}{2\sigma^2}$. Our claim holds.

\begin{align*}
        \mathcal{D}_{KL}(\mathbb{P}^0\| \mathbb{P}^i) &\leq \frac{\epsilon^2}{2\sigma^2}\sum_{t=1}^T \mathbb{E}_{\mathbb{P}^0}[\mathbb{I}\{S_{t}^{i}\neq S_{\rho^*(t)}^{i}\}]\\
        &\leq \frac{\epsilon^2}{2\sigma^2}\mathbb{E}_{\mathbb{P}^0}\left[\sum_{t=1}^T \mathbb{I}\{S_{t}^{i}\neq S_{\rho^*(t)}^{i}\}\right]\\
        &\leq \frac{\epsilon^2}{2\sigma^2}\mathbb{E}_{\mathbb{P}^0}\left[\sum_{t=1}^T\sum_{\rho^*(t)\leq s<t} \mathbb{I}\{S_{s-1}^{i}\neq S_{s}^{i}\}\right]\\
        &=\frac{\epsilon^2}{2\sigma^2}\mathbb{E}_{\mathbb{P}^0}\left[\sum_{s=1}^T\sum_{t\in \mbox{cut}_{\rho^*}(s)} \mathbb{I}\{S_{s-1}^{i}\neq S_{s}^{i}\}\right]\\&\leq \frac{\epsilon^2}{2\sigma^2}\omega(\rho^*)\mathbb{E}_{\mathbb{P}^0}\left[\sum_{s=1}^T\mathbb{I}\{S_{s-1}^{i}\neq S_{s}^{i}\}\right]\\
        &\leq \frac{\epsilon^2}{2\sigma^2}\omega(\rho^*)\mathbb{E}_{\mathbb{P}^0}\left[\frac{8\epsilon T_i}{1-8\epsilon}\right]=\frac{4\epsilon^3 \omega(\rho^*)\mathbb{E}_{\mathbb{P}^0}[T_i]}{(1-8\epsilon)\sigma^2}
\end{align*}
The last inequality is because $\sum_{s=1}^T\mathbb{I}\{S_{s-1}^{i}\neq S_{s}^{i}\}$ is the times of state switching, and we bound it using \lem{switches}. By Pinsker's inequality, we have
\begin{align*}
    \mathcal{D}_{TV}^\mathcal{F}(\mathbb{P}^0,\mathbb{P}^i)\leq \sqrt{\frac{1}{2}\mathcal{D}_{KL}(\mathbb{P}^0\| \mathbb{P}^i)}\leq (\epsilon/\sigma)\sqrt{\frac{2\omega(\rho^*)\epsilon \mathbb{E}_{\mathbb{P}^0}[T_i]}{1-8\epsilon}}
\end{align*}

\end{proof}

\begin{proof}[Proof of \thm{thm3}.]
By Yao's minimax priniple, we only need to prove that there is a distribution of loss sequence such that any deterministic learner can only achieve $\tilde{\Omega}(K^{1/3}T^{2/3})$ expected regret. The distribution of loss sequence is just the construction described in Section 4.1. 

We first consider the untruncated regret
\begin{align*}
\hat{R}_{T}=\sum_{t=1}^{T}l_{t}'(a_{t})-\min_{y\in\mathcal{A}}\sum_{t=1}^{T}l_{t}'(y)
\end{align*}
Let $T_{i}$ denote the times learner chooses arm $i$. Then $\hat{R}_{T}=\epsilon(T-T_{Z})$ for $Z>0$. Take expectation, we have
\begin{align*}
    \mathbb{E}[\hat{R}_{T}]&=\frac{1}{K+1}\sum_{i=1}^{K}\mathbb{E}[\hat{R}_{T}|Z=i]\\
    &=\frac{1}{K+1}\sum_{i=1}^{K}\epsilon\left(T-\mathbb{E}[T_{i}|Z=i]\right)
\end{align*}
Since 
\begin{align*}
    \left|\mathbb{E}[T_i| Z=i]-\mathbb{E}[T_i| Z=0]\right|&\leq \sum_{t=1}^T \left|\mathbb{P}^0(a_t = i) - \mathbb{P}^i(a_t=i)\right|\\ &\leq T \cdot \mathcal{D}_{TV}(\mathbb{P}^0,\mathbb{P}^i)
\end{align*}
we have
\begin{align*}
    \mathbb{E}[\hat{R}_{T}]&\geq\frac{1}{K+1}\sum_{i=1}^{K}\epsilon\left(T-\mathbb{E}[T_{i}|Z=0]-T\cdot\mathcal{ D}_{TV}^{\mathcal{F}}(\mathbb{P}^{0},\mathbb{P}^{i})\right)\\
    &=\frac{\epsilon KT}{K+1}-\frac{\epsilon}{K+1}\mathbb{E}\left[\sum_{i=1}^K T_i \Big{|}  Z=0\right]-\frac{\epsilon T}{K+1}\sum_{i=1}^K \mathcal{ D}_{TV}^{\mathcal{F}}(\mathbb{P}^{0},\mathbb{P}^{i})\\
    &=\frac{\epsilon(K-1)T}{K+1}-\frac{\epsilon T}{K+1}\sum_{i=1}^{K}\mathcal{D}_{TV}^{\mathcal{F}}(\mathbb{P}^{0},\mathbb{P}^{i})\\
    &\geq\frac{\epsilon (K-1)T}{K+1}-\frac{\epsilon T(\epsilon/\sigma)}{K+1}\sum_{i=1}^K\sqrt{\frac{2\omega(\rho^*)\cdot\epsilon \mathbb{E}_{\mathbb{P}^0}[T_i]}{1-8\epsilon}}
\end{align*}
The last inequality is due to \lem{TV}. Then use the concavity of square root, we have
\begin{align*}
    \mathbb{E}[\hat{R}_{T}]&\geq\frac{\epsilon (K-1)T}{K+1}-\frac{\epsilon T(\epsilon/\sigma)}{K+1}\sqrt{\frac{2\omega(\rho^*)\cdot\epsilon K\mathbb{E}_{\mathbb{P}^0}\left[\sum_{i=1}^K T_i\right]}{1-8\epsilon}}\\
    &\geq\frac{\epsilon (K-1)T}{K+1}-\frac{\epsilon T(\epsilon/\sigma)}{K+1}\sqrt{\frac{2\omega(\rho^*)\cdot\epsilon K T}{1-8\epsilon}}\\
     &\geq\frac{\epsilon}{2}T-2\epsilon^{5/2}\sigma^{-1}K^{-1/2}T^{3/2}\sqrt{\log T + 1}
\end{align*}
When $T\geq 3$, $\sqrt{\log T +1}\leq \sqrt{2\log T}$, and set $\epsilon = c K^{1/3} T^{-1/3}\log^{-1} T$, $\sigma=\frac{1}{16\sqrt{2}}\log^{-1} T$, where $c$ is a constant to be determined later, we get
\begin{align*}
    \mathbb{E}[\hat{R}_{T}]&\geq \left(\frac{c}{2}-64c^{5/2}\right)K^{1/3}T^{2/3}\log^{-1} T
\end{align*}
Let $A$ be the event that all $l_t(a)\in \left[\frac{1}{2},1\right]$. 
\[\mathbb{P}(A)\geq \mathbb{P}\left(\max_{t\in [T]}|W_t|\leq \frac{1}{4}-\epsilon \right)\]
Since $T$ tends to infinity, we can assume $\epsilon\leq \frac{1}{8}$. By \lem{range}, we have
\[\mathbb{P}(A)\geq \mathbb{P}\left(\max_{t\in [T]}|W_t|\leq \frac{1}{4}-\epsilon \right)\geq  \mathbb{P}\left(\max_{t\in [T]}|W_t|\leq \frac{1}{8} \right)\geq  \frac{2}{3} \]
Now we consider the true regret $R_T$
\begin{align*} 
    \mathbb{E}[R_{T}]&=\mathbb{P}(A)\mathbb{E}[R_T|A]+\mathbb{P}(A^c)\mathbb{E}[R_T|A^c]\\&\geq\mathbb{P}(A)\mathbb{E}[\hat{R}_T|A]\\&\geq\mathbb{P}(A)\left(\mathbb{E}[\hat{R}_T]-(1-\mathbb{P}(A))\mathbb{E}[\hat{R}_T|A^c]\right)\\&\geq\frac{2}{3}\left(\mathbb{E}[\hat{R}_T]-\frac{1}{3}\epsilon T\right)\\&=\frac{2}{3}\left(\frac{c}{6}-64c^{5/2}\right)K^{1/3}T^{2/3}\log^{-1} T
\end{align*}
Set $c = \frac{1}{64}$, $\mathbb{E}[R_{T}]\geq \frac{1}{1536}K^{1/3}T^{2/3}\log^{-1} T$.

\end{proof}

\section{Conclusion}
In this paper, we generalize the previous works of bandits with composite anonymous delayed feedback. We consider the non-oblivious loss adversary with bounded memory and the non-oblivious delay adversary. Though the external pseudo-regret incurs $\Omega(T)$ lower bound, for policy regret, we propose a mini-batch wrapper algorithm which can convert any standard non-oblivious bandit algorithm to the algorithm which fits our setting. By applying this algorithm,  we prove a $\mathcal{O}(T^{2/3})$ policy regret bound on $K$-armed bandit and bandit convex optimization, generalizing the results of \cite{cesa2018nonstochastic} and \cite{wang2021adaptive}. We also prove the matching lower bound for $K$-armed bandit problem, and our lower bound works even when the loss sequence is oblivious.

\bibliographystyle{alphaUrlePrint}
\bibliography{NOBsim}

\newcommand{\etalchar}[1]{$^{#1}$}
\begin{thebibliography}{PBASG18}

\bibitem[ACBFS02]{auer2002nonstochastic}
Peter Auer, Nicolo Cesa-Bianchi, Yoav Freund, and Robert~E Schapire.
\newblock The nonstochastic multiarmed bandit problem.
\newblock {\em SIAM journal on computing}, 32(1):48--77, 2002.

\bibitem[ADMM18]{arora2018policy}
Raman Arora, Michael Dinitz, Teodor~V Marinov, and Mehryar Mohri.
\newblock Policy regret in repeated games.
\newblock In {\em NeurIPS}, volume 2018, pages 6732--6741, 2018.

\bibitem[ADT12]{arora2012online}
Raman Arora, Ofer Dekel, and Ambuj Tewari.
\newblock Online bandit learning against an adaptive adversary: from regret to
  policy regret.
\newblock In {\em ICML}, pages 1747--1754, 2012.

\bibitem[AHM15]{anava2015online}
Oren Anava, Elad Hazan, and Shie Mannor.
\newblock Online learning for adversaries with memory: price of past mistakes.
\newblock In {\em NeurIPS}, pages 784--792, 2015.

\bibitem[BLE17]{kernel17}
S{\'{e}}bastien Bubeck, Yin~Tat Lee, and Ronen Eldan.
\newblock \href{http://dx.doi.org/10.1145/3055399.3055403}{Kernel-based methods
  for bandit convex optimization}.
\newblock In {\em STOC}, pages 72--85, 2017.

\bibitem[BZC{\etalchar{+}}19]{bistritz2019exp3}
Ilai Bistritz, Zhengyuan Zhou, Xi~Chen, Nicholas Bambos, and Jose Blanchet.
\newblock Exp3 learning in adversarial bandits with delayed feedback.
\newblock In {\em NeurIPS}, 2019.

\bibitem[CBGM18]{cesa2018nonstochastic}
Nicolo Cesa-Bianchi, Claudio Gentile, and Yishay Mansour.
\newblock Nonstochastic bandits with composite anonymous feedback.
\newblock In {\em COLT}, pages 750--773, 2018.

\bibitem[CCB20]{cella2020stochastic}
Leonardo Cella and Nicol{\`o} Cesa-Bianchi.
\newblock Stochastic bandits with delay-dependent payoffs.
\newblock In {\em AISTATS}, pages 1168--1177. PMLR, 2020.

\bibitem[CMR14]{chapelle2014simple}
Olivier Chapelle, Eren Manavoglu, and Romer Rosales.
\newblock Simple and scalable response prediction for display advertising.
\newblock {\em ACM Transactions on Intelligent Systems and Technology (TIST)},
  5(4):1--34, 2014.

\bibitem[Cov99]{cover1999elements}
Thomas~M Cover.
\newblock {\em Elements of information theory}.
\newblock John Wiley \& Sons, 1999.

\bibitem[DDKP14]{dekel2014bandits}
Ofer Dekel, Jian Ding, Tomer Koren, and Yuval Peres.
\newblock Bandits with switching costs: T$^{2/3}$ regret.
\newblock In {\em STOC}, pages 459--467, 2014.

\bibitem[FKM05]{flaxman2005online}
Abraham~D Flaxman, Adam~Tauman Kalai, and H~Brendan McMahan.
\newblock Online convex optimization in the bandit setting: gradient descent
  without a gradient.
\newblock In {\em SODA}, pages 385--394, 2005.

\bibitem[GJ21]{gyorgy2021adapting}
Andras Gyorgy and Pooria Joulani.
\newblock Adapting to delays and data in adversarial multi-armed bandits.
\newblock In {\em ICML}, pages 3988--3997, 2021.

\bibitem[GNSA10]{gergely2010online}
Andr{\'a}s~Gy{\"o}rgy Gergely~Neu, Csaba Szepesv{\'a}ri, and Andr{\'a}s Antos.
\newblock Online markov decision processes under bandit feedback.
\newblock In {\em NeurIPS}, 2010.

\bibitem[HKR16]{heidari2016tight}
Hoda Heidari, Michael Kearns, and Aaron Roth.
\newblock Tight policy regret bounds for improving and decaying bandits.
\newblock In {\em IJCAI}, pages 1562--1570, 2016.

\bibitem[JGS13]{joulani2013online}
Pooria Joulani, Andras Gyorgy, and Csaba Szepesv{\'a}ri.
\newblock Online learning under delayed feedback.
\newblock In {\em ICML}, pages 1453--1461, 2013.

\bibitem[JKLV19]{jaghargh2019consistent}
Mohammad Reza~Karimi Jaghargh, Andreas Krause, Silvio Lattanzi, and Sergei
  Vassilvtiskii.
\newblock Consistent online optimization: Convex and submodular.
\newblock In {\em AISTATS}, pages 2241--2250, 2019.

\bibitem[KS06]{kocsis2006discounted}
Levente Kocsis and Csaba Szepesv{\'a}ri.
\newblock Discounted ucb.
\newblock In {\em 2nd PASCAL Challenges Workshop}, volume~2, 2006.

\bibitem[LCG19]{li2019bandit}
Bingcong Li, Tianyi Chen, and Georgios~B Giannakis.
\newblock Bandit online learning with unknown delays.
\newblock In {\em AISTATS}, pages 993--1002, 2019.

\bibitem[LTM17]{lei2017actor}
Huitian Lei, Ambuj Tewari, and Susan~A Murphy.
\newblock An actor-critic contextual bandit algorithm for personalized mobile
  health interventions.
\newblock {\em arXiv:1706.09090}, 2017.

\bibitem[PBASG18]{pike2018bandits}
Ciara Pike-Burke, Shipra Agrawal, Csaba Szepesvari, and Steffen Grunewalder.
\newblock Bandits with delayed, aggregated anonymous feedback.
\newblock In {\em ICML}, pages 4105--4113, 2018.

\bibitem[QK15]{quanrud2015online}
Kent Quanrud and Daniel Khashabi.
\newblock Online learning with adversarial delays.
\newblock In {\em NeurIPS}, pages 1270--1278, 2015.

\bibitem[SHM{\etalchar{+}}16]{silver2016mastering}
David Silver, Aja Huang, Chris~J Maddison, Arthur Guez, Laurent Sifre, George
  Van Den~Driessche, Julian Schrittwieser, Ioannis Antonoglou, Veda
  Panneershelvam, Marc Lanctot, et~al.
\newblock Mastering the game of go with deep neural networks and tree search.
\newblock {\em Nature}, 529(7587):484--489, 2016.

\bibitem[ST11]{saha2011improved}
Ankan Saha and Ambuj Tewari.
\newblock Improved regret guarantees for online smooth convex optimization with
  bandit feedback.
\newblock In {\em AISTATS}, pages 636--642. JMLR Workshop and Conference
  Proceedings, 2011.

\bibitem[TCBS19]{thune2019nonstochastic}
Tobias~Sommer Thune, Nicol{\`o} Cesa-Bianchi, and Yevgeny Seldin.
\newblock Nonstochastic multiarmed bandits with unrestricted delays.
\newblock In {\em NeurIPS}, pages 6541--6550, 2019.

\bibitem[VBW15]{villar2015multi}
Sof{\'\i}a~S Villar, Jack Bowden, and James Wason.
\newblock Multi-armed bandit models for the optimal design of clinical trials:
  benefits and challenges.
\newblock {\em Statistical science: a review journal of the Institute of
  Mathematical Statistics}, 30(2):199, 2015.

\bibitem[WWH21]{wang2021adaptive}
Siwei Wang, Haoyun Wang, and Longbo Huang.
\newblock Adaptive algorithms for multi-armed bandit with composite and
  anonymous feedback.
\newblock In {\em AAAI}, volume~35, pages 10210--10217, 2021.

\bibitem[ZS20]{zimmert2020optimal}
Julian Zimmert and Yevgeny Seldin.
\newblock An optimal algorithm for adversarial bandits with arbitrary delays.
\newblock In {\em AISTATS}, pages 3285--3294, 2020.

\end{thebibliography}
\newpage

\appendix
\section*{Appendix}
\section{Proof of \texorpdfstring{\lem{trunc}}{}}

\begin{lemma}[measure theoretic Jensen's inequality]
\label{lem:jensen}
Let $(\Omega,A,\mu)$ be a probability space. If $g_1, g_2, \cdots, g_n$ is real valued function that is $\mu$-integrable, and $\varphi$ is a convex function on $\mathbb{R}^n$. Then 
\begin{align*}
    \varphi\left(\int_{\Omega}g_1 d\mu,\cdots,\int_{\Omega}g_n d\mu\right) \leq \int_{\Omega} \varphi \left(g_1,\cdots,g_n\right)d\mu
\end{align*}
\end{lemma}

\begin{proof}
    Let $\mathbf{v}_0$ be $\left(\int_{\Omega}g_1 d\mu,\int_{\Omega}g_2 d\mu,\cdots,\int_{\Omega}g_n d\mu\right)$, $\mathbf{v}_0$ is a point in $\mathbb{R}^n$. Since $\varphi$ is convex, we can find $\varphi$'s subgradient $\mathbf{g}_0$ at $\mathbf{v}_0$ and $\varphi(\mathbf{v})\geq \langle \mathbf{g}_0,\mathbf{v}\rangle+\varphi(\mathbf{v_0})-\langle \mathbf{g}_0,\mathbf{v}_0\rangle$. Then we have 
    \begin{align*}
        \int_{\Omega} \varphi \left(g_1,g_2,\cdots,g_n\right)d\mu&\geq \int_{\Omega} \langle \mathbf{g}_0,\mathbf{v}\rangle+\varphi(\mathbf{v_0})-\langle \mathbf{g}_0,\mathbf{v}_0\rangle d\mu \\&=\langle \mathbf{g}_0,\mathbf{v}_0\rangle+\varphi(\mathbf{v}_0) -\langle \mathbf{g}_0,\mathbf{v}_0\rangle\\&=\varphi(\mathbf{v}_0)\\&=\varphi\left(\int_{\Omega}g_1 d\mu,\int_{\Omega}g_2 d\mu,\cdots,\int_{\Omega}g_n d\mu\right)
    \end{align*}
\end{proof}

\begin{proof}[Proof of \lem{trunc}]
Consider dominating measure $\lambda=\mathbb{P}+\mathbb{Q}$. Obviously, both $\mathbb{P}$ and $\mathbb{Q}$ are absolutely continuous with respect to $\lambda$ so they have density with respect to $\lambda$ by Radon-Nikodym theorem. Let $p(x),q(x)$ be these density functiones. Let $\phi_p(x),\phi_q(x)$ be the density functions of gaussian measures which have same variance $\sigma$ and means $\mu_p,\mu_q$ respectively, with Lebesgue measure as the dominating measure. Let $\Phi_p(x)=\int_{-\infty}^{x}\phi_p(x) dx$ and $\Phi_q(x)=\int_{-\infty}^{x}\phi_q(x)dx$ be the cumulative distribution function of Gaussians. Then
\begin{align*}
    p(x)=\begin{cases}
\frac{\phi_{p}(x)}{\phi_{p}(x)+\phi_{q}(x)} & x\in(a,b)\\
\frac{\Phi_{p}(a)}{\Phi_{p}(a)+\Phi_{q}(a)} & x=a\\
\frac{1-\Phi_{p}(b)}{2-\Phi_{p}(b)-\Phi_{q}(b)} & x=b
\end{cases}\quad  q(x)=\begin{cases}
\frac{\phi_{q}(x)}{\phi_{p}(x)+\phi_{q}(x)} & x\in(a,b)\\
\frac{\Phi_{q}(a)}{\Phi_{p}(a)+\Phi_{q}(a)} & x=a\\
\frac{1-\Phi_{q}(b)}{2-\Phi_{p}(b)-\Phi_{q}(b)} & x=b
\end{cases}
\end{align*}
Calculate $\mathcal{D}_{KL}(\mathbb{P}\|\mathbb{Q})$ with dominating measure $\lambda$
    \begin{align*}
        \mathcal{D}_{KL}(\mathbb{P}\|\mathbb{Q})&=\int_{\mathbb{R}} p(x)\log\frac{p(x)}{q(x)}d\lambda(x)\\&=\int_{a}^{b}\phi_{p}(x)\log\frac{\phi_{p}(x)}{\phi_{q}(x)}dx+\Phi_{p}(a)\log\frac{\Phi_{p}(a)}{\Phi_{q}(a)}+(1-\Phi_{p}(b))\log\frac{1-\Phi_{p}(b)}{1-\Phi_{q}(b)}
    \end{align*}
Note that KL divergence is convex, by \lem{jensen}
\begin{align*}
    \mathcal{D}_{KL}(\mathbb{P}\|\mathbb{Q})&\leq\int_{a}^{b}\phi_{p}(x)\log\frac{\phi_{p}(x)}{\phi_{q}(x)}dx+\int_{-\infty}^{a}\phi_{p}(x)\log\frac{\phi_{p}(x)}{\phi_{q}(x)}dx+\int_{b}^{\infty}\phi_{p}(x)\log\frac{\phi_{p}(x)}{\phi_{q}(x)}dx\\&=\int_{-\infty}^{\infty}\phi_{p}(x)\log\frac{\phi_{p}(x)}{\phi_{q}(x)}dx=\frac{|\mu_{p}-\mu_{q}|^{2}}{2\sigma^{2}}
\end{align*}
\end{proof}
\end{document}